\documentclass{article}

\usepackage[noblocks]{authblk}
\usepackage[square]{natbib}
\usepackage{amsmath}
\usepackage{amsthm}
\usepackage{amssymb}
\usepackage{algorithm}
\usepackage{algorithmic}
\usepackage[hidelinks, breaklinks=true]{hyperref}
\usepackage{graphicx}
\usepackage{subcaption}
\usepackage{verbatim}

\newtheorem{proposition}{Proposition}

\title{Change Point Detection by Cross-Entropy Maximization}

\author{Aur\'{e}lien Serre \qquad Didier Ch\'{e}telat \qquad Andrea Lodi\\\textit{CERC, Polytechnique Montr\'{e}al}

\texttt{\{aurelien.serre, didier.chetelat, andrea.lodi\}@polymtl.ca}}

\begin{document}

\maketitle

\begin{abstract}
Many offline unsupervised change point detection algorithms rely on minimizing a penalized sum of segment-wise costs. We extend this framework by proposing to minimize a sum of discrepancies between segments. In particular, we propose to select the change points so as to maximize the cross-entropy between successive segments, balanced by a penalty for introducing new change points. We propose a dynamic programming algorithm to solve this problem and analyze its complexity. Experiments on two challenging datasets demonstrate the advantages of our method compared to three state-of-the-art approaches.
\end{abstract}

\section{Introduction}
\label{sec:introduction}

Change point detection is the problem of dividing a time series into statistically homogeneous segments. It has numerous applications in fields as varied as manufacturing \citep{PAGE1954}, speech processing \citep{Chowdhury2012}, climate science \citep{Reeves2007}, bioinformatics \citep{Oudre,Haynes2017,Schroder2018,Bosc2003, Hocking2013, Keogh2001} and finance \citep{Killick2012a}. The problem can be offline, where segmentation can be performed after seeing all the data, or online, where the time series must be segmented as data streams in. The problem can also be unsupervised or supervised, depending on whether expert segmentations are available for training. In this article we focus on the offline unsupervised problem.

Offline unsupervised change point detection is usually based on statistical testing or fitting a piecewise model. As there is a wide variety of approaches, it is difficult to give a general summary. However, importantly, most methods can be reinterpreted as finding an optimal segmentation $\mathcal{T}$ with change points $1 = \tau_0 < \tau_1 < \cdots < \tau_{m} < \tau_{m+1} = T$ of a time-series $x = x_{1:T}$ by solving an optimization problem
\begin{align}
\label{eq:classic-cost}
    \underset{\mathcal{T}}{\min}\;\sum_{i=1}^{m+1} c(x_{\tau_{i-1}:\tau_{i}}) + \beta\,\text{pen}(\mathcal{T}),
\end{align}
where $c(x)$ is a cost function on each segment and $\text{pen}(\mathcal{T})$ is a segmentation-specific penalty \citep{Truong2019}. The costs usually come from negative log-likelihoods (for piecewise models) or test statistics (in hypothesis testing).

An advantage of problems of the form \eqref{eq:classic-cost} is that algorithms have been developed that can solve them for any cost function, exactly or approximately, and very efficiently. Moreover, these algorithms tend to scale well: for modern large, high-dimensional time series, they are often the only realistic option. Nevertheless, the framework is also limiting: for example, when costs represent negative log-likelihoods, we are fitting a piecewise model where segments are statistically independent, an assumption that is often untenable. Ideally, one would like to retain the computational advantages of \eqref{eq:classic-cost} while moving towards more expressive (but less scalable) methods that do not solve \eqref{eq:classic-cost}, such as hidden semi-Markov models \citep{Baum1966}. A simple option would be to consider a generalized sum-of-costs problem, say
\begin{align}
\label{eq:pair-cost}
    \underset{\mathcal{T}}{\min}\;\sum_{i=1}^{m} c(x_{\tau_{i-1}:\tau_{i}}, x_{\tau_{i}:\tau_{i+1}}) + \beta\,\text{pen}(\mathcal{T}),
\end{align}
where $c(x,y)$ would now be a cost function depending on pairs of consecutive segments. This would increase the expressive power of the methods while remaining efficiently solvable by appropriate modifications of standard algorithms. Unfortunately, it does not appear obvious how to automatically generalize standard single-segment costs to this case. 

In this work, we propose that in the common case of negative log-likelihood costs, the negative empirical cross-entropy $c(x,y) = \text{nce}(y\,||\,x)$ could be an appropriate generalization. We combine it with standard choices: we penalize the number of change points $|\mathcal{T}|$, select the penalty hyperparameter $\beta$ by a BIC criterion, and adapt \cite{Jackson2005a}'s Optimal Partitioning algorithm to the pairwise cost case, a variant we call OTAWA. Results from the experimental section show that this leads to improvements of performance on two difficult real-world, labeled datasets against three standard methods solving \eqref{eq:classic-cost}: Optimal Partitioning, Window Sliding \citep[Algorithm 3]{Truong2019} and Binary Segmentation \citep{Scott1974}.

The OTAWA algorithm we develop in this work is in fact completely general and can solve exactly any problem of the form \eqref{eq:pair-cost} with $\text{pen}(\mathcal{T})=|\mathcal{T}|$. Thus the approach is in principle not restricted to cross-entropy costs, if alternatives could be derived in future work (e.g. analogues of test statistic costs).

This article is divided as follows.
Section~\ref{sec:literature} summarizes the literature leading to this work.
Section~\ref{sec:methodology} details the OTAWA algorithm we propose.
Section~\ref{sec:complexity_analysis} analyses the computational complexity of the resulting algorithm.
Section~\ref{sec:approximations} proposes constraints that can be added to the segmentation search so as to reduce computational cost.
Finally, Section~\ref{sec:experiments} details the experimental results.

\section{Previous Work}\label{sec:literature}

Many excellent surveys of offline unsupervised change point detection have been published recently \citep{ Tartakovsky2014, Aminikhanghahi2017, Truong2019}. We thus only give a brief overview of methods solving Equation~\eqref{eq:classic-cost}, which includes most methods in the literature. We can characterize them as a combination of a segment cost function, a penalty function for controlling the number of change points and an algorithm to solve the optimization problem.

\subsection{Cost function}%
\label{sub:segment_cost}

The costs used in the literature are often negative log-likelihoods of time series models. The earliest example of such a choice is the work of \citet{PAGE1954}, which assumes a normal distribution with fixed variance, with corresponding cost $c(x_{\tau_{i}:\tau_{i+1}}) = \sum_{t=\tau_{i}}^{\tau_{i+1}} \lVert x_t - \bar{x}_{\tau_i:\tau_{i+1}} \rVert_2^2$. Since then, many other models have been proposed, such as i.i.d. Poisson \citep{Ko2015} and autoregressive models \citep{Chakar2017}. Alternatively, methods based on hypothesis testing are also popular, such as \citet{Zou2014}, where the costs derive from the test statistic. These methods are often non-parametric.

\subsection{Penalty}%
\label{sec:penalty}

Directly minimizing the sum of segment costs generally leads to overfitting, meaning that the number of change points tends to be overestimated \citep{Haynes2017a}. A penalty term $\text{pen}(\mathcal{T})$ helps avoid overfitting by penalizing segmentations with too many change points. A very common choice is a penalty term linear in the number of change points, of the form $\text{pen}(\mathcal{T}) = |\mathcal{T}|$ \citep{Killick2012}.

\subsection{Optimization Algorithm}%
\label{sub:optimization}

In offline unsupervised change point detection, the problem \eqref{eq:classic-cost} is often solved to optimality. Many approaches are based on dynamic programming, such as the Optimal Partitioning (OP) algorithm proposed by~\cite{Jackson2005a} for penalty terms linear in the number of change points $\text{pen}(\mathcal{T}) = | \mathcal{T} |$, or the Segment Neighborhood (SN) algorithm proposed by~\cite{Auger1989}. The complexity of these dynamic programming algorithms can be improved under certain conditions by using pruning rules, such as in the work of~\cite{Haynes2017a}, \citet{Rigaill2015} and \citet{Maidstone2017}.

In cases where segmentations need to be computed in a very short amount of time, approximate algorithms have also been developed that converge faster at the expense of accuracy. Binary Segmentation \citep{Scott1974} is an example of such an algorithm. It sequentially adds a single change point to the current segmentation greedily until a stopping criterion is met. This approximate algorithm is usually faster than its exact counterpart because it only requires estimating the location of a single change point at a time, which is a much simpler problem than the global multiple change point detection problem. A related algorithm, Bottom Up \citep{Keogh2001}, works similarly but starts with many candidate change points and then sequentially removes them greedily. Finally, the Window Sliding algorithm \citep[Algorithm 3]{Truong2019} is another popular approximate algorithm for change point detection.
It computes the score 
\begin{align}
    d(t) = c(x_{t - L:t + L}) - [c(x_{t - L:t}) + c(x_{t:t + L})]
    \label{eq:ws-score}
\end{align} 
for every time-index $t$ such that $L \leq t \leq T - L$, which aims to capture the discrepancy between the statistical properties of two adjacent windows of length $L$ located on each side of $t$. The locations of the change points is then obtained using a peak detection algorithm on those scores.

\section{Methodology}\label{sec:methodology}

We propose to extend upon previous work by solving problem \eqref{eq:pair-cost} with the empirical cross-entropy as cost measure, the number of change points as penalty, a dynamic programming algorithm as exact solving method and the BIC criterion for selecting the penalty hyperparameter. We detail each choice in turn in this section.

\subsection{Cost Function}
\label{sub:cost_function}

The cross-entropy is a measure of discrepancy between two distributions with densities $f_x$ and $f_y$, and is defined by
\begin{align*}
	\text{CE}(f_y\,\Vert\,f_x) = - \int f_y(t) \log f_x(t) \, \mathrm{d}t.
\end{align*}
We propose to use the negative of the empirical analogue of this measure as cost function. Namely, given two successive segments $x_{\tau_{i-1}:\tau_{i}}$ and $x_{\tau_{i}:\tau_{i+1}}$, one can fit a statistical time series model $f_\theta$, such as an i.i.d. Gaussian or an autoregressive model, by maximum likelihood on the prior segment $x_{\tau_{i-1}:\tau_{i}}$. Denote by $\hat{\theta}^{i, \text{ML}}$ the resulting maximum likelihood estimate of the parameters of the model fitted on the $i^\text{th}$ segment: we will from now on refer to the $f_{\hat{\theta}^{i, \text{ML}}}$ as the ``segment models''. Then our cost functions are the average log-likelihood of the prior segment model on the subsequent segment,
\begin{align}
	c(x_{\tau_{i-1}:\tau_{i}}, x_{\tau_{i}:\tau_{i+1}}) = \text{nce}(x_{\tau_{i-1}:\tau_{i}}, x_{\tau_{i}:\tau_{i+1}}) 
	\notag\\\qquad
\equiv \frac{1}{\tau_{i+1}-\tau_{i}}\sum_{j=\tau_{i}}^{\tau_{i+1}}\log f_{\hat{\theta}^{i, \text{ML}}}\big(x_{j}|x_{\tau_{i}:j}\big).
\label{methodology:cost}
\end{align}

Minimizing this measure has then the effect of looking for points $\tau_i$ such that the cross-entropy of the distribution of $x_{\tau_{i-1}:\tau_{i}}$ and $x_{\tau_{i}:\tau_{i+1}}$ is maximized, that is, such that the distributions are as different as possible. Those are presumably points where there are abrupt changes in the underlying statistical distributions.

\subsection{Penalty}
\label{sub:penalty}

Fitting the optimization problem \eqref{eq:pair-cost} without any penalty will tend to produce solutions that vastly overestimate the number of change points: that is, we need to prevent overfitting. We follow the classical choice
\begin{align}
\text{pen}(\mathcal{T}) = |\mathcal{T}|,
\label{methodology:penalty}
\end{align}
the number of change points. 

\subsection{Optimization Algorithm}%
\label{sub:optimization_algorithm}

With the choices of $c(\cdot,\cdot)$ and $\text{pen}(\cdot)$ as in Equations \eqref{methodology:cost}--\,\eqref{methodology:penalty}, the optimization problem \eqref{eq:pair-cost} becomes
\begin{align}
\label{eq:optimization-problem}
    \underset{\mathcal{T}}{\min}\;V(\mathcal{T}, x) \equiv \sum_{i=1}^{m} \text{nce}(x_{\tau_{i-1}:\tau_{i}}, x_{\tau_{i}:\tau_{i+1}}) + \beta|\mathcal{T}|.
\end{align}
We propose to solve this problem exactly using dynamic programming. This is possible because we can regard it as an optimal control problem, where we must successively decide on the location of each change point $\tau_{i+1}$ in order, incurring the cost $c(x_{i-1:i}, x_{i:i+1})+\beta$ after each action.

From this point of view we can derive an algorithm, which we call Optimal Two Adjacent Windows Algorithm (OTAWA), that will efficiently find the optimal solution to optimization problem \eqref{eq:optimization-problem} for a given hyperparameter $\beta$. A pseudocode description is given as Algorithm \ref{algo:PenalizedOTAWA}. A proof of correctness as well as a complexity analysis will be given in Section \ref{sec:complexity_analysis}.

\begin{algorithm}[t]
    \caption{OTAWA} \label{algo:PenalizedOTAWA}
    \begin{algorithmic}
    \REQUIRE Time-series $x_{1:T}$, penalty hyperparameter $\beta$
        \STATE \textbf{Declare} $G$ a real-valued $(T-1)\times(T-1)$ array
        \STATE \textbf{Declare} $S$ a set-valued $(T-1)\times(T-1)$ array
        \FOR{$u=2, \dots, T-1$}
            \STATE \textbf{Init} $G[1,u] = 0$
            \STATE \textbf{Init} $S[1,u] = \{ 1, u \}$
        \ENDFOR
        \FOR{$s=2, \dots, T-1$}
            \FOR{$r=1, \dots, s-1$}
                \STATE \textbf{Fit} model $f_\theta$ on $x_{r:s}$, 
            \ENDFOR
        \ENDFOR
        \FOR{$s=2, \dots, T-1$}
            \FOR{$t=s+1, \dots, T$}
                \FOR{$r=1, \dots, s-1$}
                    \STATE \textbf{Compute} $c(x_{r:s}, x_{s:t})$ following \eqref{methodology:cost}
                \ENDFOR
                \STATE $r^* = \arg\min_{1 \leq r < s} \{ G[r,s] + c(x_{r:s}, x_{s:t}) - \beta \}$
                \STATE $G[s,t] = G[r^*,s] + c(x_{r^*:s}, x_{s:t}) - \beta$
                \STATE $S[s,t] = S[r^*,s] \cup {t}$
            \ENDFOR
        \ENDFOR
        \STATE $t^* = \arg\min_{1 < t < T} G[t, T]$
    \ENSURE $\mathcal{T}=S[t^*,T]$
    \end{algorithmic}
\end{algorithm}

\subsection{Hyperparameter Selection}
\label{sub:hyperparameter}
OTAWA will solve optimization problem \eqref{eq:optimization-problem} for a given penalty hyperparameter $\beta$, which controls the final number of change points. Since we are in an unsupervised context, however, it is not obvious how to select the number of change points. We propose a model selection procedure based on the following argument.

We can regard solving our procedure as fitting a piecewise model on the time series $x$. Namely, each segment is modeled by its segment model $f_{\hat{\theta}^{i, \text{ML}}}$ fitted by maximum likelihood, while the change points are fitted to maximize the cross-entropy between the successive segments. If the segmentation is reasonable, then this piecewise model should generalize well, in the sense that given a new time series realization $y$ from the same distribution as our training time series $x$, our piecewise model fitted on $x$ should give high likelihood to $y$. So a reasonable criterion for selecting $\beta$ is to choose it so as to minimize generalization error. 

A classic procedure for selecting hyperparameters that minimize generalization error is the Bayesian Information Criterion (BIC) \citep{Schwarz1978}. Let $\mathcal{T}^*(\beta)$ be an optimal segmentation for the optimization problem \eqref{eq:optimization-problem} with hyperparameter $\beta$. The BIC of the corresponding piecewise model is
{\setlength{\abovedisplayskip}{3pt}
\setlength{\belowdisplayskip}{3pt}
\begin{align*}
    \text{BIC}(\beta) = &-2\sum_{i=1}^{m}\sum_{j=\tau^*_{i-1}}^{\tau^*_i}\log f_{\hat{\theta}^{i, \text{ML}}}(x_j|x_{\tau^*_{i-1}:j-1})
\\[-7pt]&\hspace{60pt}
    +\log(T)\sum_{i=1}^{m}p_i
\end{align*}}
where $p_i$ is the number of parameters of the segment model $f_{\hat{\theta}^{i, \text{ML}}}$. For example, univariate i.i.d.\ Gaussian models have two parameters and univariate autoregressive models of order $p$ have $p$ parameters.

One approach to minimize the BIC is to compute it on a grid of $\beta$ and selecting the $\beta$ leading to the smallest BIC. Alternatively, one can use the CROPS algorithm \citep{Haynes2017a} to efficiently find the best segmentation for all $\beta$'s in a desired range $[\beta_{\min}, \beta_{\max}]$. If we denote by $m_{\min}$ and $m_{\max}$ the number of change points of the optimal segmentation corresponding to $\beta_{\min}$ and $\beta_{\max}$, then CROPS will find the optimal segmentation in the range in worst-case time $O(m_{\min}-m_{\max})$.

\section{Analysis}\label{sec:complexity_analysis}

In this section we perform a theoretical analysis of our proposed change point detection algorithm. We first show that OTAWA indeed solves the desired optimization problem.
\vspace{5pt}
\begin{proposition}
    Let $\mathcal{T}(x, \beta)$ be the output of OTAWA for a given time series $x$ and penalty hyperparameter $\beta$. Then $\mathcal{T}(x, \beta)$ solves optimization problem \eqref{eq:optimization-problem} exactly.
\end{proposition}
\begin{proof}
For any integers $1<i<j$, let $\text{Segment}(i,j)=\{(\tau_0,\dots,\tau_m, i, j)\,\vert\,0\leq m<i, 1=\tau_0<\tau_1<\dots<\tau_m<i\}$ the set of all segmentations of $x_{1:j}$ whose second-to-last change point is $i$, and let
\begin{align*}
    &\mathcal{T}^*_{i, j} \in\underset{\mathcal{T}\in\text{Segment}(i,j)}{\arg\min}
    V\big(\mathcal{T}, x_{1:j}\big), \\
    &G(i,j) = V\big(\mathcal{T}^*_{i, j}, x_{1:j}\big).
\end{align*}
Moreover, for any $1<j$ let $G(1,j)=0$. Now take any integers $1<s<t$ and write $\mathcal{T}^*_{s,t}=(\tau_0,\dots,\tau_m, s, t)$. If $m=0$, then 
\begin{align}
G(s,t) &= c(x_{1:s}, x_{s:t}) + \beta
= G(1,s)+c(x_{1:s}, x_{s:t}) + \beta
\notag\\
&\geq \min_{1 \leq r < s} \{G(r,s) + c(x_{r:s},x_{s:t}) + \beta\}
\label{eq:correctness-upper-bound-1}
\end{align} 
as $G(1,s)=0$. Otherwise, if $m\geq1$, let $\mathcal{T}=(\tau_0,\dots,\tau_m, s)\in \text{Segment}(\tau_m,s)$. Then for any other $\mathcal{S}=(\sigma_0,\dots,\sigma_n,\tau_m, s)\in\text{Segment}(\tau_m,s)$ we have
\begin{align*}
&V(\mathcal{T}, x_{1:s}) + c(x_{\tau_m:s}, x_{s:t}) + \beta 
= V(\mathcal{T}^*_{s,t}, x_{1:t})
\\&\qquad
\leq V\big((\sigma_0,\dots,\sigma_n, \tau_m, s, t), x_{1:t}\big)
\\&\qquad
= V(\mathcal{S}, x_{1:s}) + c(x_{\tau_m:s}, x_{s:t}) + \beta,
\end{align*}
so $V(\mathcal{T}, x_{1:s})\leq V(\mathcal{S}, x_{1:s})$ for any $S\in \text{Segment}(\tau_m,s)$, that is, $G(\tau_m,s) = V(\mathcal{T}, x_{1:s})$. Thus
\begin{align}
    G(s, t) &= G(\tau_m,s) + c(x_{\tau_m:s}, x_{s:t}) + \beta
\notag\\&
\geq\min_{1 \leq  r < s} \{G(r,s) + c(x_{r:s},x_{s:t}) + \beta\}.
\label{eq:correctness-upper-bound-2}
\end{align}

On the other hand, take any integers $1 \leq r < s$. If $r=1$, then
\begin{align}
G(1,s) &+ c(x_{1:s},x_{s:t}) + \beta 
= c(x_{1:s},x_{s:t}) + \beta 
\notag\\&
= V\big((1,s,t), x_{1:t}\big)
\geq G(s,t)
\label{eq:correctness-lower-bound-1}
\end{align}
since $G(1,s)=0$. Else if $r>1$, take any $\mathcal{T}^*_{r, s} \in{\arg\min}_{\mathcal{T}\in\text{Segment}(r,s)} V\big(\mathcal{T}, x_{1:s}\big)$ and write it as $\mathcal{T}^*_{r,s}=(\tau_0,\dots,\tau_m, r,s)$. Then observe that
\begin{align}
&G(r,s) + c(x_{r:s},x_{s:t}) + \beta 
\notag\\&\qquad = V\big((\tau_0,\dots,\tau_m, r,s,t), x_{1:t}\big)\geq G(s,t).
\label{eq:correctness-lower-bound-2}
\end{align}
Thus by combining the cases of Equations \eqref{eq:correctness-lower-bound-1}--\eqref{eq:correctness-lower-bound-2} and taking a minimum we find
\begin{align}
\min_{1 \leq r < s} \{G(r,s) + c(x_{r:s},x_{s:t}) + \beta\} \geq G(s,t).
\label{eq:correctness-lower-bound}
\end{align}
Then combining Equation \eqref{eq:correctness-lower-bound} with Equations \eqref{eq:correctness-upper-bound-1}--\eqref{eq:correctness-upper-bound-2} yields that for any $1<s<t$,
\begin{align} \label{eq:PersoOPRecurrence}
 G(s,t) = \min_{1 \leq r < s} \{G(r,s) + c(x_{r:s},x_{s:t}) + \beta\},
\end{align}
a Bellman-type equation. But this means we can compute $\mathcal{T}^*_{i,j}$ and $G(i,j)$ recursively from Equation \eqref{eq:PersoOPRecurrence}, starting from the initial condition $G(1,t)=0$, and once those are computed the optimal segmentation can be found by $\mathcal{T}^*=\arg\min_{\mathcal{T}^*_{i,T}} V(\mathcal{T}^*_{i,T}, x_{1:T})$. This is Algorithm \ref{algo:PenalizedOTAWA}.
\end{proof}

Next, we show that OTAWA has at least quartic complexity in the total number of timesteps of the time series.
\vspace{5pt}
\begin{proposition}
\label{prop:complexity}
Let $h(t)$ denote the worst-case time complexity of fitting a segment model $f_\theta$ on a time series of $t$ timesteps, and assume $h$ is monotone increasing. Then OTAWA has $O(T^2h(T) + T^4)$ worst-case time complexity.
\end{proposition}
\begin{proof}
The initialization loop takes $O(T)$ iterations, with inner operations of constant time complexity. The model fitting loop takes $O(T^2)$ iterations, and the model fitting itself takes at worst $h(T)$ time. Computation of a cost $c(x_{r:s}, x_{s:t})$ following Equation \eqref{methodology:cost} takes $O(T)$ worst-case time, and there are $O(T^3)$ costs to be computed. Computation of $r^*$ takes $O(T)$ worst-case time while the other operations in the loop are constant, and there are $O(T^2)$ iterations. Finally, computation of the final argmin takes $O(T)$ time. This yields a total worst-case time complexity of $O(T^2h(T) + T^4)$.
\end{proof}

Since in practice one must find a good $\beta$, the complexity is actually higher. For example, using the CROPS algorithm as mentioned in Section \ref{sub:hyperparameter} to search for a segmentation with no more than $M$ change points lead to an overall $O(MT^2h(T) + MT^4)$ complexity.

\section{Constrained Segmentations}
\label{sec:approximations}

For many time series and choices of local models, optimization problem \eqref{eq:optimization-problem} can be solved exactly, but for other large time series this is prohibitive.  A typical solution is to restrict ourselves to searching among a subset of segmentations, since often one has a certain tolerance as to where a change point might lie. For example, one can impose a minimal distance between consecutive change points, or require that change points lie on a grid, e.g. as implemented in the change point detection library \texttt{ruptures} \citep{Truong2018ruptures}.

Formally, the first constraint imposes that $\tau_{i+1} - \tau_i \geq S \; \forall i = 0, \ldots, m$ for some $S\in\mathbb{N}$. This is reasonable as knowledge regarding how close changes can be is often available, and also as a minimal amount of observations is required within each segment in order for the local models to be trained efficiently, hence for the cross-entropy estimate to be reasonable. The second constraint imposes that $\tau_i\in R\,\mathbb{N} \; \forall i = 0, \ldots, m$ for some $R\in\mathbb{N}$. This is reasonable again when change points must be found only within a certain precision.

Many change point detection algorithms can be accommodated to search only for segmentations satisfying these constraints, and OTAWA is no exception. Indeed, in our case, all that is needed is to restrict the range of loops over $r$, $s$, and $t$ in Algorithm \ref{algo:PenalizedOTAWA} appropriately. In such a case, following the same reasoning as in Proposition \ref{prop:complexity} yields that the resulting worst-case time complexity of the algorithm is reduced to $O\big(T(T-S)h(T)/R^2 + T^2(T-S)^2/R^3\big)$. In practice the computational gains are significant.

\section{Experiments}\label{sec:experiments}

In this section, the performance of OTAWA is compared to three other methods from the literature on two real-world datasets.

\subsection{Setup}

We compare OTAWA against the Optimal Partitioning (OP) algorithm \citep{Jackson2005a}, a state-of-the-art exact method solving optimization problem \eqref{eq:classic-cost}. We use as single-segment cost the negative log-likelihood of a model fitted on the segment, and we use the number of change points as penalty. The optimization problem is solved using the PELT \citep{Killick2012} implementation in the \texttt{ruptures} library \citep{Truong2018ruptures}. We also compare against two approximate algorithms : the Window Sliding (WS) algorithm detailed in \citet[Algorithm 3]{Truong2019}, and the Binary Segmentation (BS) algorithm \citep{Scott1974}, implemented in \texttt{ruptures}. Both algorithms are based on the same single-segment cost as OP.

In order to evaluate the performance of the different methods, we use two real-world datasets for which the positions of the change points have been labeled manually. The four methods that we compare are unsupervised. The knowledge of the positions of the change points is only used a posteriori in order to evaluate the accuracy of the estimated segmentations.
All four methods also rely on segment-wise statistical models, and for a given dataset the same model is used with all methods, for comparison purposes.

We evaluate performance using the following five metrics. The Annotation Error is the absolute difference between the estimated number of change points $\widehat{m}$ and the true number $m^*$, $\textsc{AnnotationError}(\widehat{\mathcal{T}}, \mathcal{T}^*) = | m^* - \widehat{m} |$. $\textsc{F1-Score}(\mathcal{T}^*, \widehat{\mathcal{T}})$ and $\textsc{Precision}(\mathcal{T}^*, \widehat{\mathcal{T}})$ are the F1-score and precision derived from defining a detection radius $r > 0$, and considering a true change point as detected if a change point has been estimated within $r$ time-indices of its location. $\textsc{Hausdorff}(\mathcal{T}^*, \widehat{\mathcal{T}})$ is the Hausdorff distance between the subsets of $\{1, \dots, T\}$ corresponding to the true segmentation $\mathcal{T}^*$ and the estimated segmentation $\widehat{\mathcal{T}}$, seen as sets of integers. The Mean Distance metric is the mean over every true change points of the distance to the closest estimated change point,
$\textsc{MeanDistance}(\mathcal{T}^*, \widehat{\mathcal{T}}) = \frac{1}{|\mathcal{T}^*|}\sum_{t^* \in \mathcal{T}^*} \min_{\hat{t} \in \widehat{\mathcal{T}}} |\hat{t} - t^*|$. Finally, $\textsc{RandIndex}(\mathcal{T}^*, \widehat{\mathcal{T}})$, initially introduced by \cite{Rand1971} for evaluating clustering methods, is the proportion of pairs of time-indices that are either in the same segment according to both $\mathcal{T}^*$ and $\widehat{\mathcal{T}}$ or in different segments according to both $\mathcal{T}^*$ and $\widehat{\mathcal{T}}$.

As OP and BS come from the production-quality library \texttt{rutpures}  while our implementation of OTAWA and WS are proofs of concept, we do not report running times. Nonetheless, even with our development implementations, all algorithms took under 3 minutes to run on each dataset. OTAWA was the slowest in both cases, as expected from the complexity analysis of Section \ref{sec:complexity_analysis}.

\begin{figure}[t]
 \centering
 \includegraphics[width=\textwidth]{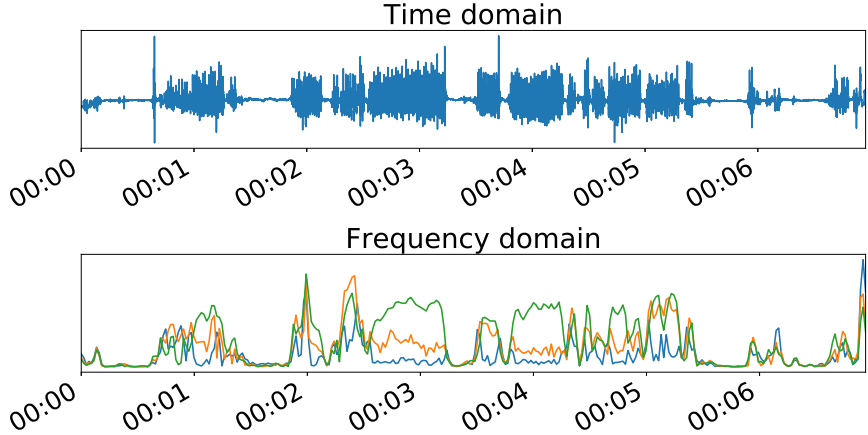}
 \caption{The Human Activities dataset. Top: acceleration measurements across time. Bottom: same signal after a STFT, only 3 of the 23 variables are displayed for visibility.}
 \label{fig:STFT_smoothing}
\end{figure}

\begin{figure}[t!]
    \centering
	\begin{subfigure}[b]{0.45\textwidth}
		\includegraphics[width=\textwidth]{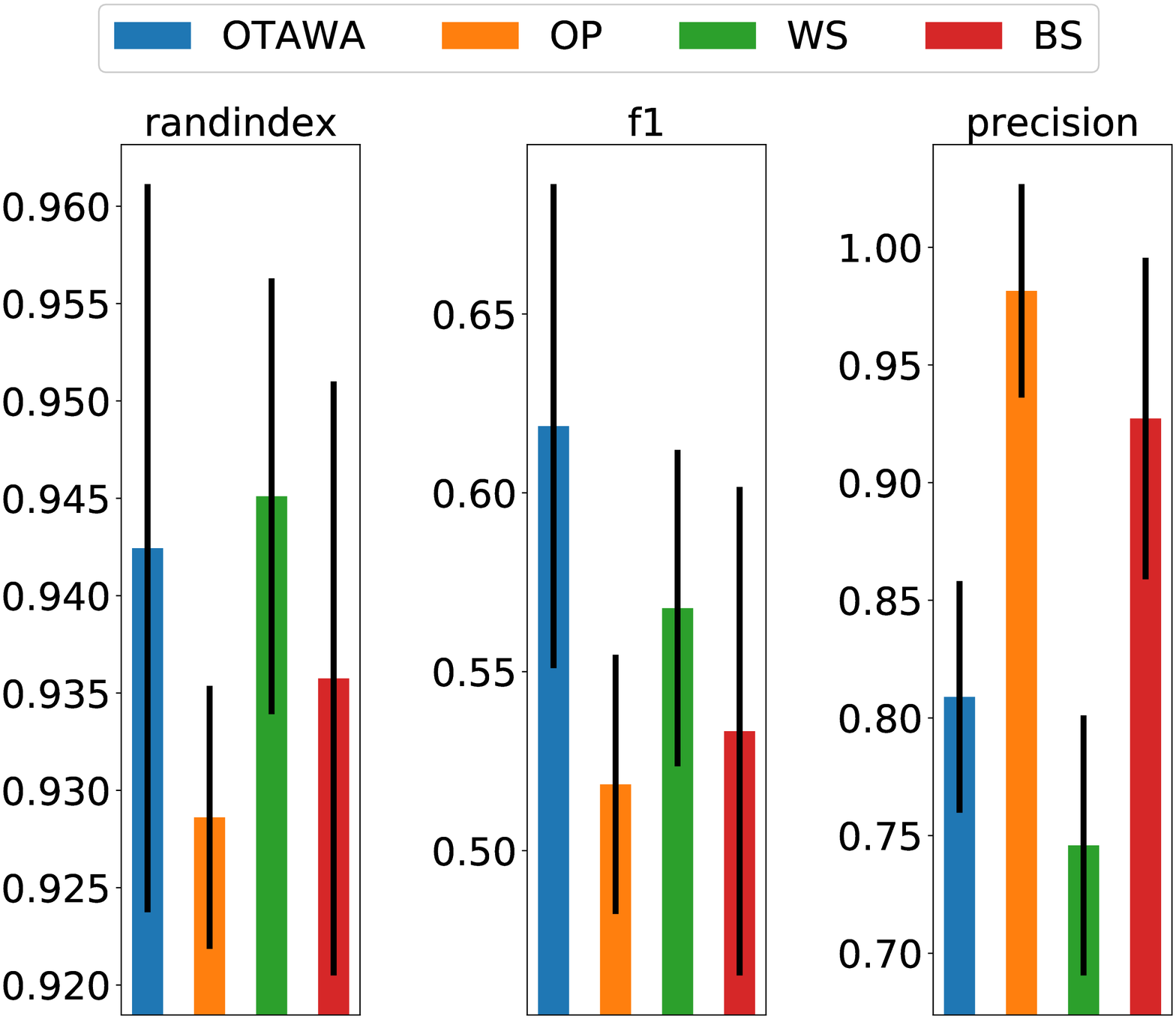}
		\caption{Results for the \textsc{RandIndex}, \textsc{F1-Score} and \textsc{Precision} metrics. (Higher is better.)}
	\end{subfigure}
	\begin{subfigure}[b]{0.45\textwidth}
		\includegraphics[width=\textwidth]{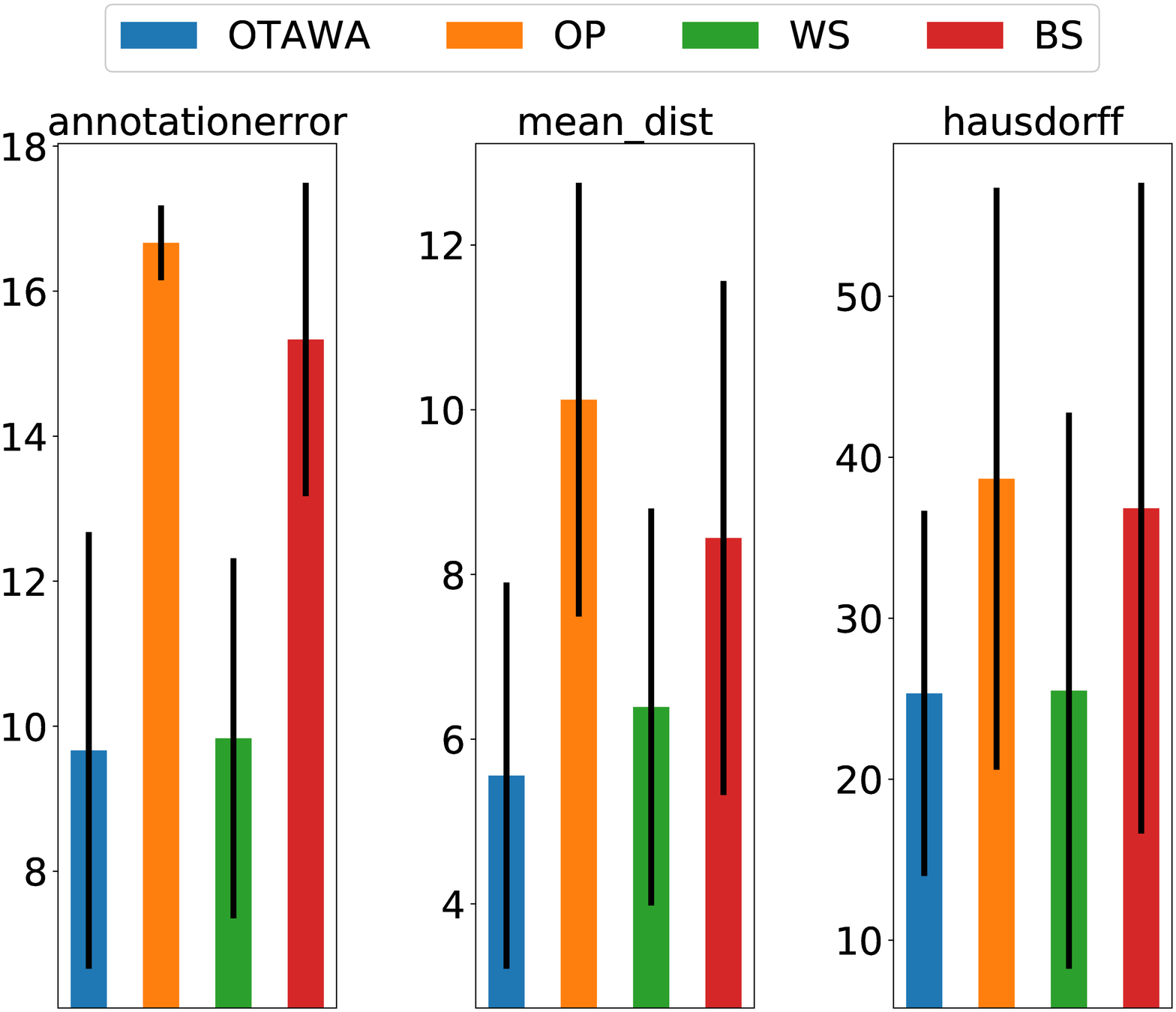}
		\caption{Results for the \textsc{AnnotationError}, \textsc{MeanDistance} and \textsc{Hausdorff} metrics. (Lower is better.)}
	\end{subfigure}
	\caption{Results for the OTAWA, Optimal Partitioning (OP),  Window Sliding (WS) and Binary Segmentation (BS) algorithms on the Human Activities dataset. The center value and the error bars represent the mean and standard deviation over the six time series of the dataset.}
	\label{fig:HASC_results}
\end{figure}

\subsection{Human Activities Dataset} \label{subsec:HAD}

The Human Activities dataset \citep{Kawaguchi2011} contains measurements acquired by a device fixed on the waist of a person while performing different activities such as walking or going up a staircase. The task is to segment the time series into the different activities in a unsupervised manner, which is a change point detection problem. Using accelerometer and gyroscope measurements along the three spacial axes yields 6 real-valued time series.

We preprocess the data by applying a short-time Fourier transform (STFT) to the time-series and clipping to the range [$0.5Hz$ -- $5Hz$] similarly to \cite{Oudre}. The STFT is performed using a window size of approximately $5s$ ($512 samples$), and an overlapping between windows of $75\%$. This yields a total of $23$ frequency bins, so that in total we have six $23$-dimensional real-valued time series with $308$ timesteps each. We normalize each variable separately so that its range lies in $[0,1]$ using min-max scaling.

Within one type of activity, it is reasonable to assume that the repetitive motion is stationary, meaning that the spectral information is stationary as well within each segment. For this reason, we assume that the observations in the time-series are i.i.d.\ and follow a Gaussian distribution with constant unknown variance, and piecewise constant mean.

With all four algorithms, we use a resolution parameter of $R = 2$ samples and a minimum segment length of $S = 8$ samples.
The Window Sliding algorithm has its window length set to $L = 10$ samples.
The \textsc{F1-Score} is computed with a detection radius of $r = 6$.

\subsubsection{Results}

Figure~\ref{fig:HASC_results} shows the mean and standard deviation of the performance obtained by each of the four methods on the six time series. We can observe that OTAWA achieves the best average performance among the four methods according to every metric except for \textsc{RandIndex}.

\begin{figure}[t!]
    \centering
	\begin{subfigure}[b]{0.45\textwidth}
		\includegraphics[width=\textwidth]{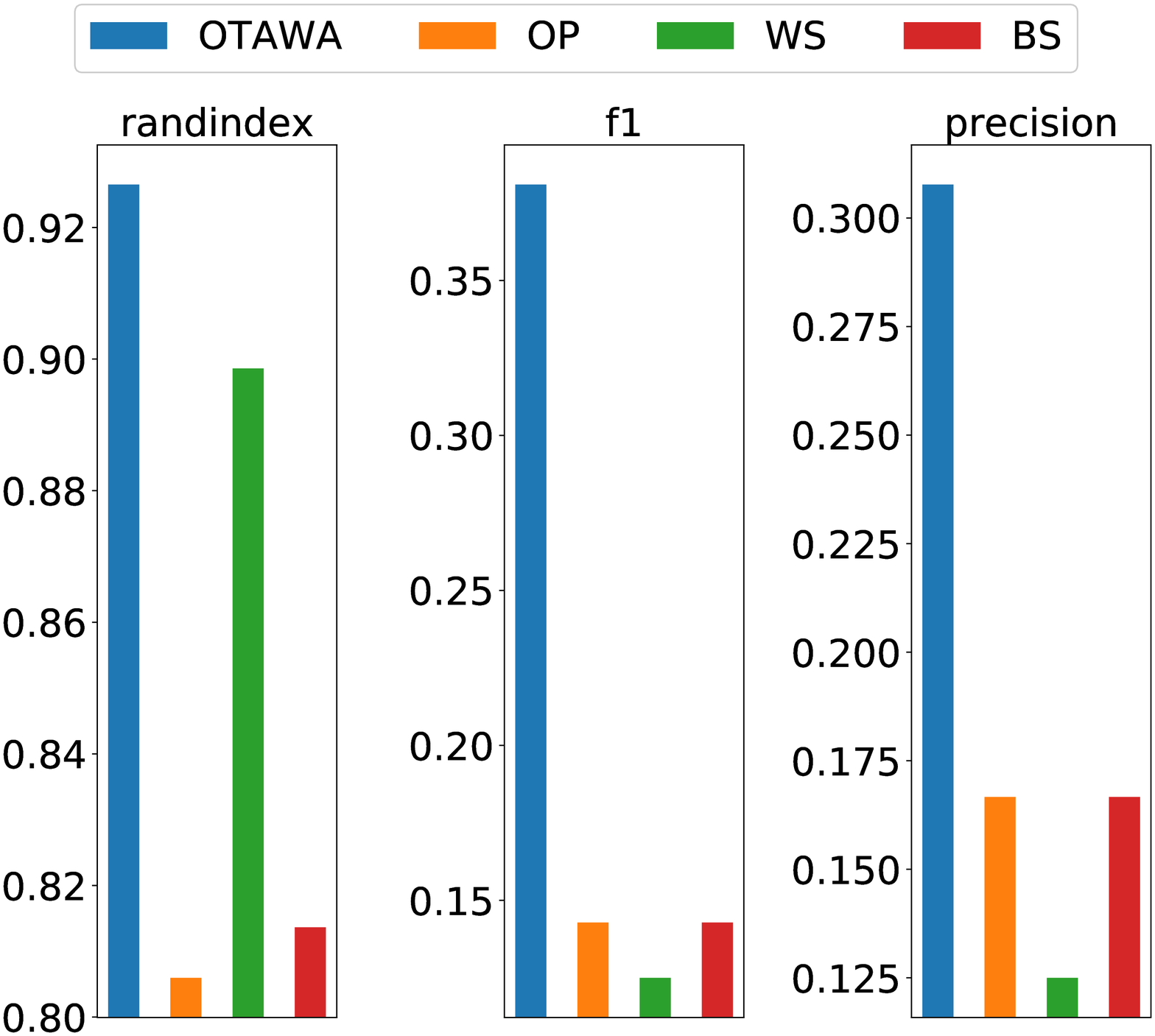}
		\caption{Results for the \textsc{RandIndex}, \textsc{F1-Score} and \textsc{Precision} metrics. (Higher is better.)}
	\end{subfigure}
	\begin{subfigure}[b]{0.45\textwidth}
		\includegraphics[width=\textwidth]{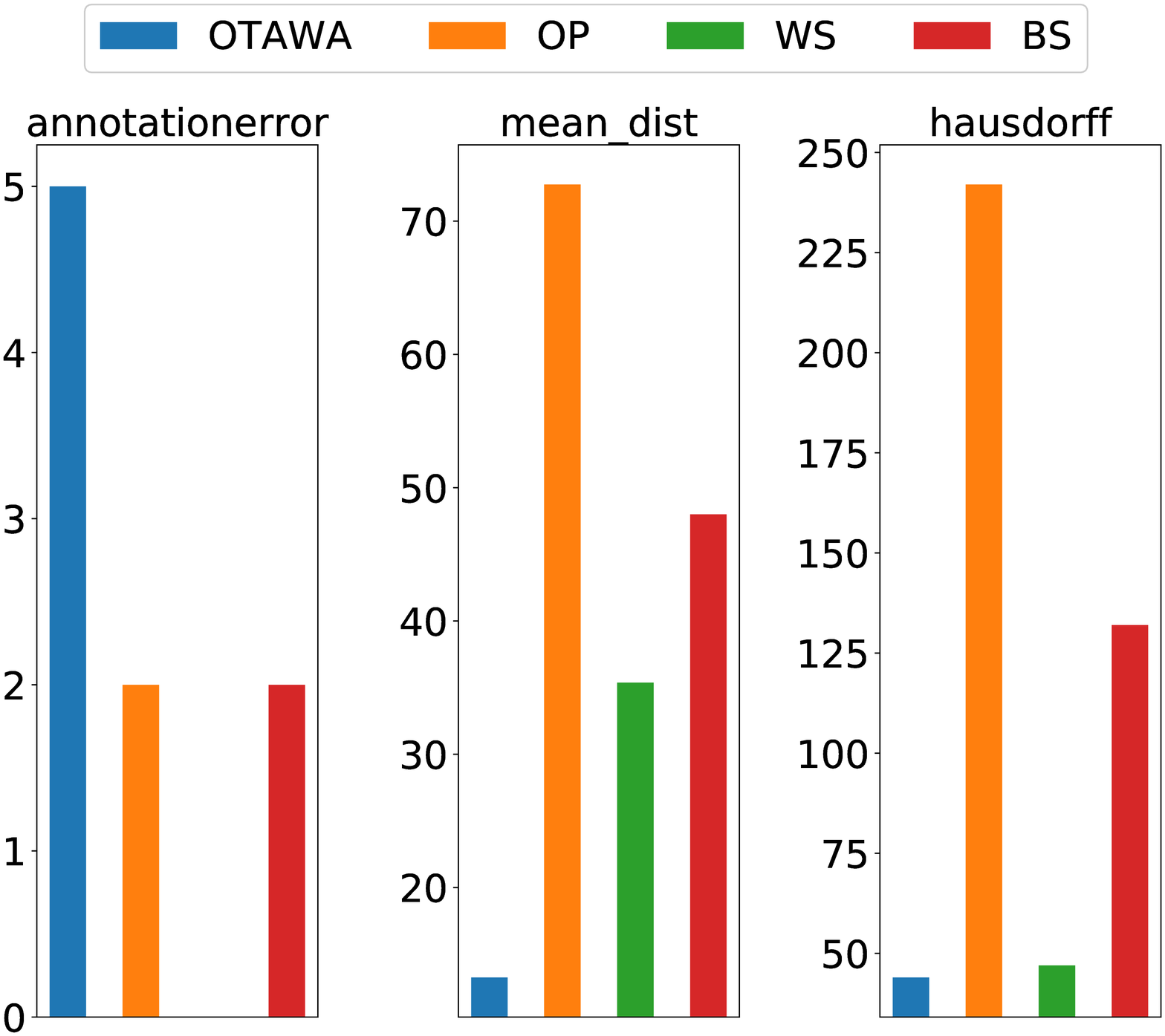}
		\caption{Results for the \textsc{AnnotationError}, \textsc{MeanDistance} and \textsc{Hausdorff} metrics. (Lower is better.)}
	\end{subfigure}
	\caption{Results for the OTAWA, Optimal Partitioning (OP), Window Sliding (WS) and Binary Segmentation (BS) algorithms on the Industrial Equipment dataset.}
	\label{fig:PREDICT_results}
\end{figure}

\subsection{Industrial Equipment Dataset}

This dataset contains measurements of two sensors acquired on an industrial equipment over a period of 716 days, to be made public upon publication of this article. The task is to retrieve in an unsupervised manner the dates of maintenance events that occurred during the recorded period. The actual dates of these events have been labeled manually, in order to evaluate the performance of the four algorithms a posteriori.  According to expert knowledge of the industrial equipment, we should expect the behavior of the equipment to slowly drift between maintenance events due to normal degradations. In contrast, maintenance events are expected to correspond to abrupt changes of the behavior, from one observation to the next. Thus, this task can be interpreted as a change point detection problem.

Stationary models such as i.i.d.\ models are not suitable here, since they would not be able to capture the drift, and we propose instead to use a vector autoregressive (VAR) model. We train it with an L1 penalty to avoid overfitting. The time-series initially contains $10399$ observations, that we reduce down to $716$ by sub-sampling at a rate of one observation per day. We also normalize each variable between $0$ and $1$ using min-max scaling.

For all four methods, we use a VAR model of order $p = 3$, estimated with L1 regularization parameter $\alpha = 10^{-2}$.
We will also use a minimum spacing of $S = 50$ samples between change points, and a resolution parameter of $R = 5$.
The Window Sliding algorithm is used with windows of size $L = 20$ samples.
The \textsc{F1-Score} and \textsc{Precision} metrics are computed with a detection radius of $r = 10$.

\subsubsection{Results}

Figure~\ref{fig:PREDICT_results} reports the performance achieved by the four algorithms on the time series. As can be seen, OTAWA outperforms the other two methods according to every metric except for \textsc{AnnotationError}.

\section{Conclusion}\label{sec:conclusion}
In this article, we proposed an extension of a standard framework for offline unsupervised change point detection. We propose to select change points so as to maximize the empirical cross-entropy between successive segments, while balancing the introduction of new change points with a penalty on the number of segments. We proposed a dynamic programming algorithm to solve this problem exactly, as well as variants on a reduced search space, and detailed experimental evidence of the improvements provided by our approach against state-of-the-art methods on two challenging datasets.

Our approach can be regarded as a extension of the standard sum-of-costs formulation to costs depending on pairs of segments, in the case where costs are derived from negative log-likelihoods. A promising line of work would be to generalize this approach to arbitrary cost functions, such as derived from parametric and non-parametric hypothesis testing.

\bibliographystyle{apalike}
\bibliography{arxiv}
\end{document}